\documentclass[twocolumn]{article}

\title{Generalization of Equilibrium Propagation to Vector Field Dynamics
}
\author{Benjamin Scellier$^1$,  Anirudh  Goyal$^1$,  Jonathan Binas$^1$, Thomas Mesnard$^2$, Yoshua Bengio$^{1\dagger}$\\
$^1$ Mila, Universit\'e de Montr\'eal\\ 
$^2$ \'Ecole Normale Sup\'erieure de Paris \\
 $^\dagger$CIFAR Senior Fellow\\
}

\usepackage[T1]{fontenc}      
\usepackage[utf8]{inputenc} 
\usepackage[english]{babel}

\usepackage{amssymb,amsthm,amsfonts,amscd}
\usepackage{amsmath}
\usepackage{mathtools}   
\usepackage{times}
\usepackage{textcomp}
\usepackage{subcaption}
\usepackage{fancyhdr}    
\usepackage{lastpage}

\usepackage[margin=3cm]{caption}
\usepackage{makeidx} 
\usepackage{graphicx}
\usepackage{algorithm}
\usepackage{algorithmic}
\usepackage{natbib}
\usepackage{caption}
\graphicspath{{.},{img/}}

\newcommand{\norm}[1]{\left\lVert #1\right \rVert}   

\newcommand \x{\mathrm x}
\newcommand \y{\mathrm y}
\newcommand \s{\mathrm s}

\makeindex

\usepackage{vmargin}
\setmarginsrb{2cm}{2cm}{2cm}{3cm}{0cm}{0.8cm}{0cm}{1cm}

\newtheorem{prop}{Proposition}
\newtheorem{thm}[prop]{Theorem}
\newtheorem{lem}[prop]{Lemma}
\newtheorem{cor}[prop]{Corollary}

\begin{document}

\maketitle


\abstract{
The biological plausibility of the backpropagation algorithm has long been doubted by neuroscientists. Two major reasons are that neurons would need to send two different types of signal in the forward and backward phases, and that pairs of neurons would need to communicate through symmetric bidirectional connections.
We present a simple two-phase learning procedure for fixed point recurrent networks that addresses both these issues.
In our model, neurons perform leaky integration and synaptic weights are updated through a local mechanism.
Our learning method generalizes Equilibrium Propagation to vector field dynamics, relaxing the requirement of an energy function.
As a consequence of this generalization, the algorithm does not compute the true gradient of the objective function,
but rather approximates it at a precision which is proven to be directly related to the degree of symmetry of the feedforward and feedback weights.
We show experimentally that our algorithm optimizes the objective function.
}


\section{Introduction}

Deep learning \citep{lecun2015deep} is the de-facto standard in areas such as
computer vision \citep{krizhevsky2012imagenet}, speech recognition \citep{hinton2012deep} and machine translation \citep{bahdanau2014neural}.
These applications deal with different types of data and have little in common at first glance.
Remarkably, all these models typically rely on the same basic principle: optimization of objective functions using the {\em backpropagation} algorithm.
Hence the question: does the cortex in the brain implement a mechanism similar to backpropagation, which optimizes objective functions?

The backpropagation algorithm used to train neural networks requires a side network for the propagation of error derivatives, which is vastly seen as biologically implausible \citep{crick-nature1989}. 
One hypothesis, first formulated by \citet{hinton1988learning}, is that error signals in biological networks could be encoded in the temporal derivatives of the neural activity and propagated through the network via the neuronal dynamics itself, without the need for a side network. Neural computation would correspond to both inference and error back-propagation.
The present work also explores this idea.

Equilibrium Propagation \citep{scellier2017equilibrium} requires the network dynamics to be derived from an energy function, enabling computation of an exact gradient of an objective function.
However, in terms of biological realism, the requirement of symmetric weights between neurons arising from the energy function (the Hopfield energy) is not desirable.
The work presented here is a generalization of Equilibrium Propagation to vector field dynamics, without the need for energy functions, gradient dynamics, or symmetric connections.

Our approach is the following.
\begin{enumerate}
	\item We start from standard models in neuroscience for the dynamics of the neuron's membrane voltage and for the synaptic plasticity (section \ref{sec:neuroscience}). In particular, unlike in the Hopfield model \citep{hopfield1984neurons}, we do not assume pairs of neurons to have symmetric connections.
	\item We then describe a supervised learning algorithm for fixed point recurrent neural networks, based on these models (sections \ref{sec:supervised-learning}-\ref{sec:equilibrium-propagation}) and with few extra assumptions. Our model assumes two phases: at prediction time (first phase), no synaptic changes occur, whereas a local update rule becomes effective when the targets are observed (second phase).
	\item Finally, we attempt to show that the proposed algorithm optimizes an objective function (section \ref{sec:theory}) -- a highly desirable property from the point of view of machine learning. We show this experimentally and we attempt to understand this theoretically, too.
\end{enumerate}


\section{Neuronal Dynamics}
\label{sec:neuroscience}

\begin{figure*}[h!]
\begin{center}
\begin{subfigure}[t]{.4\textwidth}
	\centering
	\includegraphics[width=0.8\linewidth]{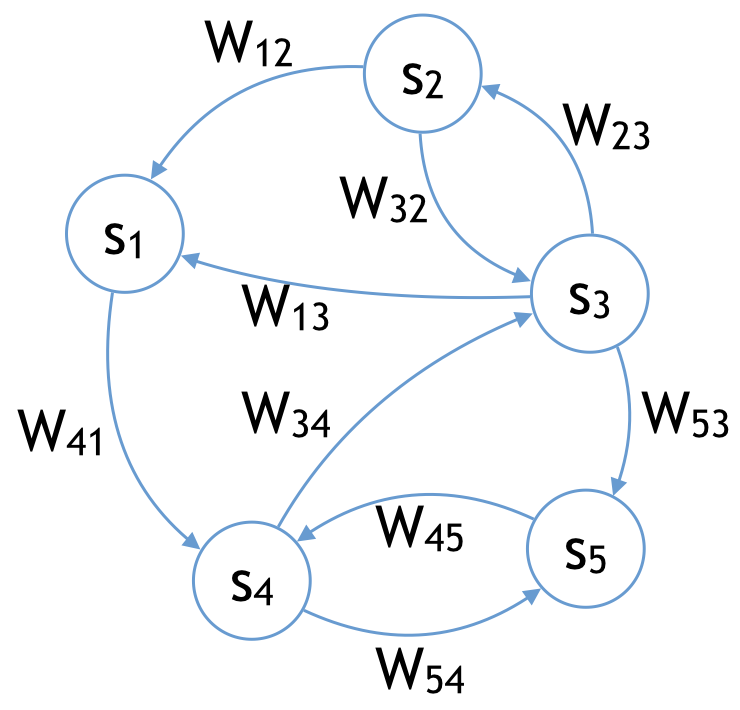}
	\caption{The network model studied here is best represented by a directed graph.}
	\label{fig:graph_directed}
\end{subfigure}
\hspace{2 em}
\begin{subfigure}[t]{.4\textwidth}
	\centering
	\includegraphics[width=0.8\linewidth]{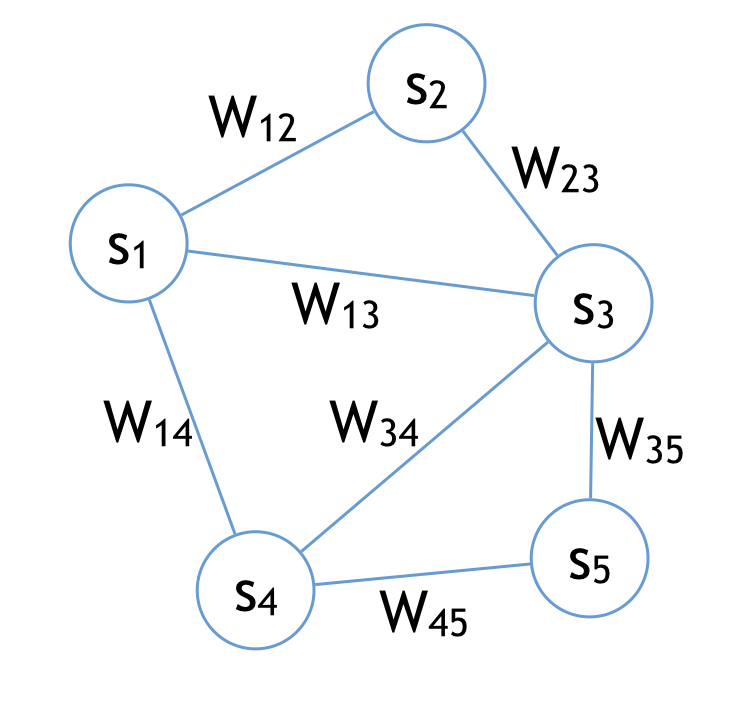}
	\caption{The Hopfield model is best represented by an undirected graph.}
	\label{fig:graph_undirected}
\end{subfigure}
\end{center}
	\caption{From the point of view of biological plausibility, the symmetry of connections in the Hopfield model is a major drawback (\ref{fig:graph_undirected}).
	The model that we study here is, like a biological neural network, a directed graph (\ref{fig:graph_directed}).}
	\label{fig:graph}
\end{figure*}

We denote by $s_i$ the membrane voltage of neuron $i$, which is continuous-valued and plays the role of a state variable for neuron $i$.
We suppose that $\rho$ is a deterministic function (nonlinear activation) that takes a scalar $s_i$ as input and outputs a scalar $\rho(s_i)$.
The scalar $\rho(s_i)$ represents the firing rate of neuron $i$.
The synaptic strength from neuron $j$ to neuron $i$ is denoted by $W_{ij}$.


\subsection{Neuron Model}

We consider the following time evolution for the membrane voltage $s_i$:
\begin{equation}
	\label{eq:leaky-integrator}
	\frac{ds_i}{dt} = \sum_j W_{ij} \rho(s_j) - s_i.
\end{equation}
Eq.~\ref{eq:leaky-integrator} is a standard point neuron model (see e.g. \citet{DayanAbbott2001})
in which neurons are seen as performing leaky temporal integration of their inputs.
We will refer to Eq.~\ref{eq:leaky-integrator} as the \textit{rate-based leaky integrator neuron model}.

Unlike energy-based models such as the Hopfield model \citep{hopfield1984neurons} that assume symmetric connections between neurons,
in the model studied here the connections between neurons are not tied.
Our model is represented by a directed graph,
whereas the Hopfield model is best regarded as an undirected graph (Figure \ref{fig:graph}).


\subsection{Plasticity Model}

We consider a simplified Hebbian update rule based on pre- and post-synaptic activity,
in which a change $ds_i$ in the post-synaptic activity causes a change $dW_{ij}$ in the synaptic strength given by
\begin{equation}
	\label{eq:stdp}
	dW_{ij} \propto \rho(s_j) ds_i.
\end{equation}

\citet{bengio2017stdp} have shown in simulations that this update rule can functionally reproduce Spike-Timing Dependent Plasticity (STDP).
STDP is considered a key mechanism of synaptic change in biological neurons \citep{Markram+Sakmann-1995,gerstner1996neuronal,markram2012spike}.
STDP is often conceived of as a spike-based process which relates the change in the synaptic weight $W_{ij}$
to the timing difference between postsynaptic spikes (in neuron $i$) and presynaptic spikes (in neuron $j$) ~\citep{bi2001synaptic}.
In fact, both experimental and computational work suggest that postsynaptic voltage, not postsynaptic spiking,
is more important for driving LTP (Long Term Potentiation) and LTD (Long Term Depression)~\citep{clopath2010voltage,lisman2010questions}.

Throughout this paper we will refer to Eq.~\ref{eq:stdp} as \textit{STDP-compatible weight change} and propose a machine learning justification for such an update rule.


\subsection{Vector Field $\mu$ in the State Space}

In this subsection we rewrite Eq.~\ref{eq:leaky-integrator} and Eq.~\ref{eq:stdp} at a higher level of abstraction.
Let $s=(s_1,s_2, \ldots)$ be the global \textit{state variable} and let $\theta = \left( W_{ij} \right)_{i,j}$ be the \textit{parameter variable} consisting of the set of synaptic weights.
We write $\mu_\theta(s)$ the vector whose components $\mu_{\theta,i}(s)$ are defined as
\begin{equation}
	\label{eq:vector-field-mu}
	\mu_{\theta,i}(s) := \sum_j W_{ij} \rho(s_j) - s_i.
\end{equation}
The vector $\mu_\theta(s)$ has the same dimension as the state variable $s$.
For fixed $\theta$, the mapping $s \mapsto \mu_\theta(s)$ is a vector field in the state space, which indicates in which direction each neuron's activity changes.
Eq.~\ref{eq:leaky-integrator} rewrites
\begin{equation}
	\label{eq:state-dynamics}
	\frac{ds}{dt} = \mu_\theta(s).
\end{equation}

Let us move on to the weight change of Eq.~\ref{eq:stdp}.
Since $\rho(s_j) = \frac{\partial \mu_{\theta,i}}{\partial W_{ij}}(s)$, the weight change can be expressed as $dW_{ij} \propto \frac{\partial \mu_{\theta,i}}{\partial W_{ij}}(s) ds_i$. Note that for all $i' \neq i$ we have $\frac{\partial \mu_{i'}}{\partial W_{ij}} = 0$ since to each synapse $W_{ij}$ corresponds a unique post-synaptic neuron $s_i$. Hence $dW_{ij} \propto \frac{\partial \mu_\theta}{\partial W_{ij}}(s) \cdot ds$.
We rewrite the STDP-compatible weight change (Eq.~\ref{eq:stdp}) in the concise form
\begin{equation}
	\label{eq:parameter-change}
	d\theta \propto \frac{\partial \mu_\theta}{\partial \theta}(s)^T  \cdot ds.
\end{equation}


\section{Fixed Point Recurrent Neural Networks for Supervised Learning}
\label{sec:supervised-learning}

We consider the supervised setting in which we want to predict a \textit{target} $\y$ given an \textit{input} $\x$.
The units of the network are split in two sets: the `input' units $\x$ whose values are always clamped, and the dynamically evolving units $s$ (the neurons activity, indicating the state of the network), which themselves include the hidden layers ($s_1$ and $s_2$ here) and an output layer ($s_0$ here), as in Figure \ref{fig:network}.
In this context the vector field $\mu_\theta$ is defined by its components $\mu_{\theta,0}$, $\mu_{\theta,1}$ and $\mu_{\theta,2}$ on $s_0$, $s_1$ and $s_2$ respectively, as follows:
\begin{align}
	\label{eq:mu0}
	\mu_{\theta,0}(\x,s) & = W_{01} \cdot \rho(s_1) - s_0, \\
	\mu_{\theta,1}(\x,s) & = W_{12} \cdot \rho(s_2) + W_{10} \cdot \rho(s_0) - s_1, \\
	\label{eq:mu2}
	\mu_{\theta,2}(\x,s) & = W_{23} \cdot \rho(\x) + W_{21} \cdot \rho(s_1) - s_2.
\end{align}
In its original definition, $\rho$ takes a scalar as input and outputs a scalar,
but here we generalize the definition of $\rho$ to act on a vector (i.e. a layer of neurons),
in which case it returns a vector of the same dimension,
operating element-wise on the coordinates of the input vector.

The neurons $s$ follow the dynamics
\begin{equation}
	\label{eq:free-dynamics}
	\frac{ds}{dt} = \mu_\theta(\x,s).
\end{equation}
Unlike in the continuous Hopfield model, here the feedforward and feedback weights are not tied,
and in general the dynamics of Eq.~\ref{eq:free-dynamics} is not guaranteed to converge to a fixed point.
However, for simplicity of presentation we assume here that the dynamics of the neurons converge to a fixed point which we denote by $s_\theta^\x$.
The fixed point $s_\theta^\x$ implicitly depends on $\theta$ and $\x$ through the relationship
\begin{equation}
	\label{eq:free-fixed-point}
	\mu_\theta \left( \x, s_\theta^\x \right) = 0.
\end{equation}

In addition to the vector field $\mu_\theta(\x,s)$, a cost function $C(\y,s)$ measures how good or bad a state $s$ is with respect to the target $\y$.
In our model, the layer $s_0$ has the same dimension as the target $\y$ and plays the role of the `output' layer where the prediction is read.
The discrepancy between the output layer $s_0$ and the target $\y$ is measured by the quadratic cost function
\begin{equation}
	\label{eq:quadratic-cost}
	C(\y,s) = \frac{1}{2}\norm{\y-s_0}^2.
\end{equation}
The prediction is then read out on the output layer at the fixed point and compared to the target $\y$.
The objective function that we aim to minimize (with respect to $\theta$) is the cost at the fixed point $s_\theta^\x$, which we write
\footnote{More generally, in order to take into account cases when the dynamics of Eq.~\ref{eq:free-dynamics} does not converge to a fixed point, we can define the objective function as the average cost of the state along the trajectory (over infinite duration).}
\begin{equation}
	\label{eq:objective-function}
	J(\x,\y,\theta) := C \left( \y,s_\theta^\x \right).
\end{equation}
\citet{Almeida87} and \citet{pineda1987generalization} proposed an algorithm known as \textit{Recurrent Backpropagation}
\footnote{`Recurrent backprop' is a special case of `backprop through time', specialized to a fixed point recurrent network (i.e. an RNN whose dynamics converges to a fixed point, like those studied here).}
to optimize $J$ by computing the gradient $\frac{\partial J}{\partial \theta}(\x,\y,\theta)$.
This algorithm, presented in Appendix \ref{appendix:rec-backprop}, is based on extra assumptions on the neuronal dynamics which make it biologically implausible.

In this paper, our approach to optimize $J$ is to give up on computing the true gradient of $J$ and, instead,
we propose a simple algorithm based on the leaky integrator dynamics (Eq.~\ref{eq:state-dynamics})
and the STDP-compatible weight change (Eq.~\ref{eq:parameter-change}).
We will show in section \ref{sec:theory} that our algorithm computes a proxy to the gradient of $J$.

Also, note that in its general formulation, our algorithm applies to any vector field $\mu_\theta(\x,s)$ and cost function $C_\theta(\y,s)$,
even when $C$ depends on $\theta$ (Appendix \ref{appendix:theorem}).


\section{Equilibrium Propagation in the Vector Field Setting}
\label{sec:equilibrium-propagation}

We describe a simple two-phase learning procedure based on the state dynamics (Eq.~\ref{eq:state-dynamics}) and the parameter change (Eq.~\ref{eq:parameter-change}).
This algorithm generalizes the one proposed in \citet{scellier2017equilibrium}. 


\subsection{Augmented Vector Field}

In its original (energy-based) version, the central idea of Equilibrium Propagation \citep{scellier2017equilibrium} is to see
the cost function $C$ (Eq.~\ref{eq:quadratic-cost}) as an `external potential energy' for the output layer $s_0$,
which can drive it towards the target $\y$.
Following the same idea we define the \textit{augmented vector field}
\begin{equation}
	\label{eq:augmented-vector-field}
	\mu_\theta^\beta(\x,\y,s) := \mu_\theta(\x,s) - \beta \frac{\partial C}{\partial s}(\y,s),
\end{equation}
where $\beta \geq 0$ is a real-valued scalar which we call the \textit{influence parameter} (or \textit{clamping factor}).
Rather than Eq.~\ref{eq:free-dynamics}, the dynamics of $s$ is more generally
\begin{equation}
	\label{eq:augmented-dynamics}
   \frac{ds}{dt} = \mu_\theta^\beta(\x,\y,s)
\end{equation}
for some value of $\beta$.
The parameter $\beta$ controls whether the output layer $s_0$ is pushed towards the target $\y$ or not, and by how much.
In particular, the dynamics of Eq.~\ref{eq:free-dynamics} corresponds to the case $\beta = 0$.

The augmented vector field can be seen as a sum of two `forces' that act on the temporal derivative of the state variable $s$.
Apart from the vector field $\mu_\theta$ that models the interactions between neurons within the network,
an `external force' $- \beta \frac{\partial C}{\partial s}$ is induced by the `external potential' $\beta C$ and acts on the output layer:
\begin{align}
	\label{eq:external-force}
	- \beta \frac{\partial C}{\partial s_0}(\y,s) & = \beta ( \y-s_0 ), \\
	- \beta \frac{\partial C}{\partial s_i}(\y,s) & = 0, \qquad \forall i \geq 1.
\end{align}
The form of Eq.~\ref{eq:external-force} suggests that
when $\beta=0$, the output layer $s_0$ is not sensitive to the target $\y$.
When $\beta > 0$, the `external force' drives the output layer $s_0$ towards the target $\y$.
The case $\beta \to \infty$ (not studied in this paper) would correspond to fully clamped output units.

Following the dynamics of Eq.~\ref{eq:augmented-dynamics}, the state variable $s$ eventually settles to a fixed point $s_\theta^\beta$ characterized by
\begin{equation}
	\mu_\theta^\beta \left( \x,\y,s_\theta^\beta \right) = 0.
\end{equation}
Note that the fixed point $s_\theta^\beta$ also depends on $\x$ and $\y$ but we omit to write the dependence to keep readable notations.
In particular for $\beta = 0$ we have $s_\theta^0 = s_\theta^\x$.

\begin{figure}[h!]
	\centering
	\captionsetup{width=.8\linewidth}
	\includegraphics[width=.6\linewidth]{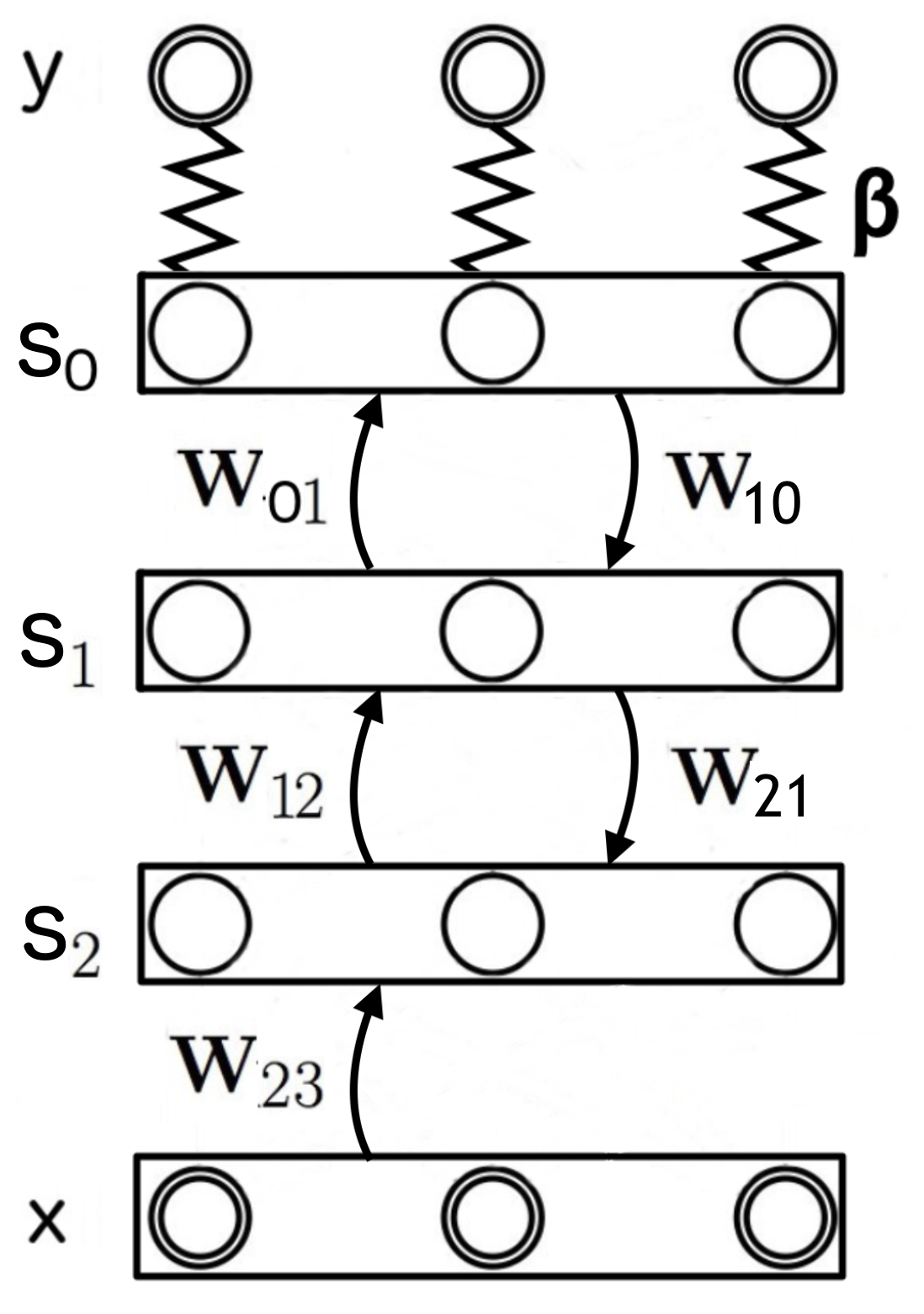}
	\caption{Graph of the network. Input $\x$ is clamped. Neurons $s$ include hidden layers $s_2$ and $s_1$,
	and output layer $s_0$ that corresponds to the layer where the prediction is read.
	Target $\y$ has the same dimension as $s_0$.
	The clamping factor $\beta$ scales the `external force' $-\beta \frac{\partial C}{\partial s}$ that attracts the output layer's state $s_0$ towards the target $\y$.}
	\label{fig:network}
\end{figure}


\subsection{Algorithm}
\label{sec:algorithm}

We propose the following two-phase learning procedure.
At prediction time (the first phase), the input units are set (clamped) to the input values $\x$, and the influence parameter $\beta$ is set to $0$.
The state variable $s$ (all the other neurons) follows the dynamics of Eq.~\ref{eq:augmented-dynamics} and settles to the first fixed point $s_\theta^0$.
During this phase, we assume that the synaptic weights are unchanged.

At training time (the second phase), the input units are still clamped and the influence parameter $\beta$ takes on a small positive value $\beta \gtrsim 0$. The state variable follows the dynamics of Eq.~\ref{eq:augmented-dynamics} for that new value of $\beta$, and the synaptic weights are assumed to follow the STDP-compatible weight change of Eq.~\ref{eq:parameter-change}.
The network eventually settles to a new nearby fixed point, denoted $s_\theta^\beta$, corresponding to the new value $\beta \gtrsim 0$.


\subsection{Backpropagation of Error Signals}

In the first phase, the influence parameter $\beta$ is equal to $0$.
The output units are `free', in the sense that they are not influenced by the target $\y$.

In the second phase, the influence parameter takes on a positive value $\beta \gtrsim 0$.
The novel `external force' $- \beta \frac{\partial C}{\partial s}(\y,s)$ in the dynamics of Eq.~\ref{eq:augmented-dynamics}
acts on the output units and drives them towards their targets (Eq.~\ref{eq:external-force}).
This force models the observation of the target $\y$:
it nudges the output units $s_0$ from their value
at the first fixed point in the direction of their targets.
Since this force only acts on the output layer $s_0$,
the other hidden layers ($s_i$ with $i>0$) are initially at equilibrium
at the beginning of the second phase.
The perturbation caused at the output layer ($s_0$) will then propagate backwards along the layers of the network ($s_1$ and $s_2$),
giving rise to `back-propagating' error signals.

A more detailed analysis of the second phase is carried out in Appendix \ref{appendix:rec-backprop} and a connection to the recurrent backpropagation algorithm \citep{Almeida87,pineda1987generalization} is established, following the ideas of \citet{scellier2017equivalence}.


\section{Optimization of the Objective Function}
\label{sec:theory}

The proposed algorithm (section \ref{sec:algorithm}) is experimentally found to optimize the objective function $J(\theta)$ (see Appendix \ref{sec:implementation}).
In this section, we attempt to understand why.


\subsection{Vector Field $\nu$ in the Parameter Space}

Our model assumes that the STDP-compatible weight change occurs during the second phase of training,
when the network's state moves from the first fixed point $s_\theta^0$ to the second fixed point $s_\theta^\beta$.
Integrating Eq.~\ref{eq:parameter-change} from $s_\theta^0$ to $s_\theta^\beta$, normalizing it by a factor $\beta$ and letting $\beta \to 0$,
we get the update rule
\begin{equation}
	\Delta \theta \propto \nu(\x,\y,\theta),
\end{equation}
where $\nu(\x,\y,\theta)$ is the vector defined as
\footnote{Recall that $s_\theta^\beta$ depends on $\x$ and $\y$.
Hence $\nu$ depends on $\x$ and $\y$ through $\left. \frac{\partial s_\theta^\beta}{\partial \beta} \right|_{\beta=0}$.}
\begin{equation}
	\label{eq:vector-field-nu}
	\nu(\x,\y,\theta) := \frac{\partial \mu_\theta}{\partial \theta} \left( \x,s_\theta^0 \right)^T \cdot \left. \frac{\partial s_\theta^\beta}{\partial \beta} \right|_{\beta=0}.
\end{equation}
The vector $\nu(\x,\y,\theta)$ has the same dimension as $\theta$.
Thus, for fixed $\x$ and $\y$, the mapping $\theta \mapsto \nu(\x,\y,\theta)$ defines a vector field in the parameter space.
We show next that $\nu(\x,\y,\theta)$ is a proxy to the gradient $-\frac{\partial J}{\partial \theta}(\x,\y,\theta)$.


\subsection{The Vector Field $\nu$ As A Proxy For The Gradient}

\begin{thm}
	\label{thm}
	The gradient $\frac{\partial J}{\partial \theta}(\x,\y,\theta)$ and the vector field $\nu(\x,\y,\theta)$ can be expressed explicitly in terms of $\mu_\theta$ and $C$:
	\begin{align*}
		\frac{\partial J}{\partial \theta}(\x,\y,\theta) & =
		- \frac{\partial C}{\partial s} \left( \y,s_\theta^\x \right) \cdot
		\left( \frac{\partial \mu_\theta}{\partial s} \left( \x,s_\theta^\x \right) \right)^{-1} \cdot
		\frac{\partial \mu_\theta}{\partial \theta} \left( \x,s_\theta^\x \right), \\
		\nu(\x,\y,\theta) & =
		\frac{\partial C}{\partial s} \left( \y,s_\theta^\x \right) \cdot
		\left( \frac{\partial \mu_\theta}{\partial s} \left( \x,s_\theta^\x \right)^T \right)^{-1} \cdot
		\frac{\partial \mu_\theta}{\partial \theta} \left( \x,s_\theta^\x \right).
	\end{align*}
\end{thm}
Theorem \ref{thm} is proved in Appendix \ref{appendix:theorem}.
The formulae show that $\nu(\x,\y,\theta)$ is related to $\frac{\partial J}{\partial \theta}(\x,\y,\theta)$ and that the angle between these two vectors is directly linked to the `degree of symmetry' of the Jacobian of $\mu_\theta$ at the fixed point $s_\theta^\x$.


\subsection{Energy Based Setting as an Idealization of the Vector Field Setting}

We say that $\theta$ is a `good parameter' if:
\begin{enumerate}
	\item for any initial state for the neurons, the state dynamics $\frac{ds}{dt} = \mu_\theta \left( \x,s \right)$ converges to a fixed point - a condition required for the algorithm to be correctly defined,
	\item the scalar product $\frac{\partial J}{\partial \theta}(\x,\y,\theta) \cdot \nu(\x,\y,\theta)$ at the point $\theta$ is negative - a desirable condition for the algorithm to optimize the objective function $J$.
\end{enumerate}

An important particular case is the energy-based setting studied in \citet{scellier2017equilibrium},
in which the vector field $\mu_\theta$ is a gradient field, i.e. $\mu_\theta(\x,s) = - \frac{\partial E_\theta}{\partial s}(\x,s)$ for some scalar function $E_\theta(\x,s)$.
In this case, the Jacobian of $\mu_\theta$ is symmetric since $\frac{\partial \mu_\theta}{\partial s} = - \frac{\partial^2 E_\theta}{\partial s^2} = \left( \frac{\partial \mu_\theta}{\partial s} \right)^T$, and by Theorem \ref{thm} we get $\nu(\x,\y,\theta) = - \frac{\partial J}{\partial \theta}(\x,\y,\theta)$.
Therefore, in this setting the set of `good parameters' is the entire parameter space --
for all $\theta$, the dynamics $\frac{ds}{dt}=-\frac{\partial E_\theta}{\partial s}(\x,s)$ converges to a fixed point (a minimum of $E_\theta(\x,s)$),
and $\nu(\x,\y,\theta) \cdot \frac{\partial J}{\partial \theta}(\x,\y,\theta) \leq 0$.


However, for the `set of good parameters' to cover a large proportion of the parameter space, it is not required that the vector field $\mu_\theta$ derives from an energy function $E_\theta$.
Indeed, experiments run on the MNIST dataset show that, when $\mu_\theta(\x,s)$ is defined as in Eq.~\ref{eq:mu0}-\ref{eq:mu2}, the objective function $J$ consistently decreases (Appendix \ref{sec:implementation}).
This means that, during training, as the parameter $\theta$ follows the update rule $\Delta \theta \propto \nu(\x,\y,\theta)$, all values of $\theta$ that the network takes are `good parameters'.

\section{Possible Implementation on Analog Hardware}
\label{sec:analog}

Our model is a continuous-time dynamical system (described by differential equations).
Digital computers are not well suited to implement such models because
they do intrinsically discrete-time computations, not continuous-time ones.
The basic way to simulate a differential equation on a digital computer is the Euler method in which time is discretized.
However the discretized dynamics is only an approximation of the true continuous-time dynamics.
The accuracy depends on the size of the discretization step.
The bigger the step size, the less acurate the simulation.
The smaller the step size, the slower the computations.

By contrast, analog hardware is ideal for implementing continuous-time dynamics such as those of leaky integrator neurons.
Previous work have proposed such implementations \citep{hertz1997nonlinear}.


\section{Related Work}
\label{related-work}

Other alternatives to recurrent back-propagation in the framework of fixed point recurrent networks were proposed
by \citet{o1996biologically} and \citet{hertz1997nonlinear}.
Their algorithms are called `Generalized Recirculation algorithm' (or `GeneRec' for short) and `Non-Linear Back-propagation', respectively.

More recently, \citet{mesnard2016towards} have adapted Equilibrium Propagation to spiking networks, bringing the model closer to real neural networks.
\citet{zenke2017superspike} also proposed a backprop-like algorithm for supervised learning in spiking networks called `SuperSpike'.
\citet{guerguiev2017towards} proposed a mechanism for backpropagating error signals in a multilayer network with compartment neurons, experimentally shown to learn useful representations.


\section{Conclusion}

Among others, two key features of the backpropagation algorithm make it biologically implausible -- the two different kinds of signals sent in the forward and backward phases and the `weight transport problem'.
In this work, we have proposed a backprop-like algorithm for fixed-point recurrent networks, which addresses both these issues.
As a key contribution, in contrast to energy-based approaches such as the Hopfield model, we do not impose any symmetry constraints on the neural connections.
Our algorithm assumes two phases, the difference between them being whether synaptic changes occur or not.
Although this assumption begs for an explanation,
neurophysiological findings suggest that phase-dependent mechanisms are involved in learning and memory consolidation in biological systems.
Synaptic plasticity, and neural dynamics in general, are known to be modulated by inhibitory neurons and dopamine release, depending on the presence or absence of a target signal \citep{fremaux2016neuromodulated,pawlak2010timing}.

In its general formulation (Appendix \ref{appendix:theorem}), the work presented in this paper is a generalization of Equilibrium Propagation \citep{scellier2017equilibrium} to vector field dynamics. This is achieved by relaxing the requirement of an energy function.
This generalization comes at the cost of not being able to compute the (true) gradient of the objective function but, rather a direction in the parameter space which is related to it.
Thereby, precision of the approximation of the gradient is directly related to the degree of symmetry of the Jacobian of the vector field.

Our work shows that optimization of an objective function can be achieved without ever computing the (true) gradient.
More thorough theoretical analysis needs to be carried out to understand and characterize the dynamics in the parameter space that optimize objective functions.
Naturally, the set of all such dynamics is much larger than the tiny subset of gradient-based dynamics.

Our framework provides a means of implementing learning in a variety of physical substrates, whose precise dynamics might not even be known exactly, but which simply have to be in the set of supported dynamics. In particular, this applies to analog electronic circuits, potentially leading to faster, more efficient, and more compact implementations.

\section*{Acknowledgments}

The authors would like to thank Blake Richards and Alexandre Thiery for feedback and discussions, as well as NSERC, CIFAR, Samsung, SNSF,
and Canada Research Chairs for funding, and Compute Canada for computing resources.


\bibliographystyle{abbrvnat}
\bibliography{biblio} 


\appendix
\part*{Appendix}


\section{Theorem \ref{thm} - General Formulation and Proof}
\label{appendix:theorem}

In this appendix, before proving Theorem \ref{thm},
we first generalize the setting of sections \ref{sec:supervised-learning} and \ref{sec:equilibrium-propagation} to the case where the cost function $C$ also depends on the parameter $\theta$.
This is the case e.g. if the cost function includes a regularization term such as $\frac{1}{2} \lambda \norm{\theta}^2$.


\subsection{General Formulation}

Recall that we consider the supervised setting in which we want to predict a \textit{target} $\y$ given an \textit{input} $\x$.
The model is specified by a \textit{state variable} $s$, a \textit{parameter variable} $\theta$,
a \textit{vector field} in the state space $\mu_\theta(\x,s)$ and a \textit{cost function} $C_\theta(s,\y)$.
\footnote{In the setting described in section \ref{sec:supervised-learning}, the cost function $C(s,\y)$ did not depend on $\theta$.}
The stable \textit{fixed point} $s_\theta^\x$ corresponding to the 'prediction' from the model is characterized by
\begin{equation}
	\mu_\theta \left( \x,s_\theta^\x \right) = 0.
\end{equation}
The objective function to be minimized is the cost at the fixed point, i.e.
\begin{equation}
	\label{eq:objective-function-general}
	J(\x,\y,\theta) := C_\theta \left( \y,s_\theta^\x \right).
\end{equation}
Traditional methods to compute the gradient of $J$ such as Recurrent Backpropagation are thought to be biologically implausible (see Appendix \ref{appendix:rec-backprop}).
Our approach is to give up on computing the gradient of $J$ and let the parameter variable $\theta$ follow a vector field $\nu$ in the parameter space which approximates the gradient of $J$.

To this end we first define the \textit{augmented vector field}
\begin{equation}
	\mu_\theta^\beta(\x,\y,s) := \mu_\theta(\x,s) - \beta \; \frac{\partial C_\theta}{\partial s}(\y,s).
\end{equation}
Here $\beta$ is a real-valued scalar which we call \textit{influence parameter}.
The corresponding fixed point $s_\theta^\beta$ is a state at which the augmented vector field vanishes, i.e.
\begin{equation}
	\mu_\theta^\beta \left( \x,\y,s_\theta^\beta \right) = 0.
\end{equation}
Under mild regularity conditions on $\mu_\theta$ and $C_\theta$, the implicit function theorem ensures that, for a fixed data sample $(\x,\y)$,
the function $(\theta,\beta) \mapsto s_\theta^\beta$ is differentiable.

Then for every $\theta$ we define the vector $\nu(\x,\y,\theta)$ in the parameter space as
\begin{equation}
	\label{eq:vector-field-nu-general}
	\nu(\x,\y,\theta) := - \frac{\partial C_\theta}{\partial \theta} \left( \y,s_\theta^\x \right)
	+ \frac{\partial \mu_\theta}{\partial \theta} \left( \x,s_\theta^\x \right)^T
	\cdot \left. \frac{\partial s_\theta^\beta}{\beta} \right|_{\beta=0}.
\end{equation}
In section \ref{sec:equilibrium-propagation} we showed how the second term on the right-hand side of Eq.~\ref{eq:vector-field-nu-general} can be estimated with a two-phase training procedure.
In the general case where the cost function also depends on the parameter $\theta$,
the definition of the vector $\nu(\x,\y,\theta)$ contains the new term
$- \frac{\partial C_\theta}{\partial \theta} \left( \y,s_\theta^\x \right)$.
This extra term can also be measured in a biologically realistic way at the first fixed point $s_\theta^0$ at the end of the first phase.
For example if $C_\theta(\y,s)$ includes a regularization term such as $\frac{1}{2} \lambda \norm{\theta}^2$,
then $\nu(\x,\y,\theta)$ will include a backmoving force $- \lambda \theta$, modeling a form of synaptic depression.


We can now reformulate Theorem \ref{thm} in a slightly more general form, when the cost function depends on $\theta$.
The gradient of the objective function and the vector field $\nu$ are equal to
\begin{align}
	\label{eq:thm-gradient}
	\frac{\partial J}{\partial \theta}(\theta) & = \frac{\partial C_\theta}{\partial \theta}
	- \frac{\partial C_\theta}{\partial s} \cdot \left( \frac{\partial \mu_\theta}{\partial s} \right)^{-1} \cdot \frac{\partial \mu_\theta}{\partial \theta}, \\
	\label{eq:thm-nu}
	\nu(\theta) & = - \frac{\partial C_\theta}{\partial \theta} + 
	\frac{\partial C_\theta}{\partial s}
	\cdot \left( \left( \frac{\partial \mu_\theta}{\partial s} \right)^T \right)^{-1}
	\cdot \frac{\partial \mu_\theta}{\partial \theta}.
\end{align}
All the factors on the right-hand sides of Eq.~\ref{eq:thm-gradient}-\ref{eq:thm-nu} are evaluated at the fixed point $s_\theta^0$.
We prove Eq.~\ref{eq:thm-gradient}-\ref{eq:thm-nu} in the next subsection.


\subsection{Proof}

In order to prove Eq.~\ref{eq:thm-gradient} and Eq.~\ref{eq:thm-nu} (i.e. Theorem \ref{thm} in a slightly more general form), we first state and prove a lemma.

\begin{lem}
	\label{lemma}
	Let $s \mapsto \mu_\theta^\beta(s)$ be a differentiable vector field, and $s_\theta^\beta$ a fixed point characterized by
	\begin{equation}
		\label{eqn:fx-pt-eq}
		\mu_\theta^\beta \left( s_\theta^\beta \right) = 0.
	\end{equation}
	Then the partial derivatives of the fixed point are given by
	\begin{equation}
		\label{eq:id-1}
	  \frac{\partial s_\theta^\beta}{\partial \theta}
	  = - \left( \frac{\partial \mu_\theta^\beta}{\partial s} \left( s_\theta^\beta \right) \right)^{-1} \cdot \frac{\partial \mu_\theta^\beta}{\partial \theta} \left( s_\theta^\beta \right)
	\end{equation}
	and
	\begin{equation}
		\label{eq:id-2}
	  \frac{\partial s_\theta^\beta}{\partial \beta}
	  = - \left( \frac{\partial \mu_\theta^\beta}{\partial s} \left( s_\theta^\beta \right) \right)^{-1} \cdot \frac{\partial \mu_\theta^\beta}{\partial \beta} \left( s_\theta^\beta \right).
	\end{equation}
\end{lem}

\begin{proof}[Proof of Lemma \ref{lemma}]
	First we differentiate the fixed point equation Eq.~\ref{eqn:fx-pt-eq} with respect to $\theta$:
	\begin{equation}
	  \label{eqn:d-dtheta}
	  \frac{d}{d\theta} \; (\ref{eqn:fx-pt-eq}) \; \Rightarrow \quad
	  \frac{\partial \mu_\theta^\beta}{\partial \theta} \left( s_\theta^\beta \right)
	  +
	  \frac{\partial \mu_\theta^\beta}{\partial s} \left( s_\theta^\beta \right) \cdot \frac{\partial s_\theta^\beta}{\partial \theta}
	  = 0.
	\end{equation}
	Rearranging the terms we get Eq.~\ref{eq:id-1}.
	Similarly we differentiate the fixed point equation Eq.~\ref{eqn:fx-pt-eq} with respect to $\beta$:
	\begin{equation}
	  \label{eqn:d-dbeta}
	  \frac{d}{d\beta} \; (\ref{eqn:fx-pt-eq}) \; \Rightarrow \quad
	  \frac{\partial \mu_\theta^\beta}{\partial \beta} \left( s_\theta^\beta \right)
	  + \frac{\partial \mu_\theta^\beta}{\partial s} \left( s_\theta^\beta \right) \cdot \frac{\partial s_\theta^\beta}{\partial \beta}
	  = 0.
	\end{equation}
	Rearranging the terms we get Eq.~\ref{eq:id-2}.
\end{proof}

Now we are ready to prove prove Eq.~\ref{eq:thm-gradient} and Eq.~\ref{eq:thm-nu}.

\begin{proof}[Proof of Theorem \ref{thm} in its general formulation]
	Let us compute the gradient of the objective function with respect to $\theta$.
	Using the chain rule of differentiation we get
	\begin{equation}
		\frac{\partial J}{\partial \theta}
		= \frac{\partial C_\theta}{\partial \theta} +
		\frac{\partial C_\theta}{\partial s} \cdot \frac{\partial s_\theta^0}{\partial \theta}.
	\end{equation}
	Hence Eq.~\ref{eq:thm-gradient} follows from Eq.~\ref{eq:id-1} evaluated at $\beta=0$.
	Similarly, the expression for the vector field $\nu$ (Eq.~\ref{eq:thm-nu}) follows
	from its definition (Eq.~\ref{eq:vector-field-nu-general}),
	the identity Eq.~\ref{eq:id-2} evaluated at $\beta=0$
	and, using Eq.~\ref{eq:augmented-vector-field}, the fact that $\frac{\partial \mu^\beta}{\partial \beta} = - \frac{\partial C}{\partial s}$.
\end{proof}

\section{Link to Recurrent Backpropagation}
\label{appendix:rec-backprop}

Earlier work have proposed various methods to compute the gradient of the objective function $J$ (Eq.~\ref{eq:objective-function-general}).
One of them is \textit{Recurrent Backpropagation}, an algorithm discovered independently by \citet{Almeida87} and \citet{pineda1987generalization}.
This algorithm assumes that neurons send a different kind of signals through a different computational path in the second phase,
which seems less biologically plausible than our algorithm.

Our approach is to give up on the idea of computing the \textit{true} gradient of the objective function.
Instead our algorithm relies only on the leaky integrator neuron dynamics (Eq.~\ref{eq:leaky-integrator}) and the STDP-compatible weight change (Eq.~\ref{eq:stdp})
and we have shown that it computes a proxy to the true gradient (Theorem \ref{thm}).

Up to now, we have described our algorithm in terms of fixed points only.
In this appendix we study the dynamics itself in the second phase when the state of the network moves from the free fixed point $s_\theta^0$ to the weakly clamped fixed point $s_\theta^\beta$.

The result established in this section is a straightforward generalization of the result proved in \citet{scellier2017equivalence}.


\subsection{Preliminary Notations}

We denote by $S_\theta^\x(\s,t)$ the state of the network at time $t \geq 0$
when it starts from an initial state $\s$ at time $t=0$ and follows the dynamics of Eq.~\ref{eq:free-dynamics}.
In the theory of dynamical systems, $S_\theta^\x(\s,t)$ is called the \textit{flow map}.
Note that as $t \to \infty$ the dynamics converges to the fixed point $S_\theta^\x(\s,t) \to s_\theta^\x$.

Next we define the \textit{projected cost function}
\begin{equation}
    \label{eq:projected-cost}
    L_\theta(\x,\y,\s,t) := C_\theta \left( \y,S_\theta^\x(\s,t) \right).
\end{equation}

This is the cost of the state projected a duration $t$ in the future, when the networks starts from $\s$ and follows the dynamics of Eq.~\ref{eq:free-dynamics}.
For fixed $\theta$, $\x$, $\y$ and $\s$, the process $\left( L_\theta(\x,\y,\s,t) \right)_{t \geq 0}$ represents the successive cost values taken by the state of the network along the dynamics when it starts from the initial state $\s$.
For $t=0$, the projected cost is simply the cost of the current state: $L_\theta(\x,\y,\s,0) = C_\theta \left( \y,\s \right)$.
As $t \to \infty$ the dynamics converges to the fixed point, i.e. $S_\theta^\x(\s,t) \to s_\theta^\x$, so the projected cost converges to the objective $L_\theta(\x,\y,\s,t) \to J(\x,\y,\theta)$.
Under mild regularity conditions on $\mu_\theta(\x,s)$ and $C_\theta(\y,s)$,
the gradient of the projected cost function converges to the gradient of the objective function as $t \to \infty$, i.e.
\begin{equation}
	\label{eq:dL-dtheta-dJ-dtheta}
	\frac{\partial L_\theta}{\partial \theta}(\x,\y,\s,t) \to \frac{\partial J}{\partial \theta}(\x,\y,\theta).
\end{equation}
Therefore, if we can compute $\frac{\partial L_\theta}{\partial \theta}(\x,\y,\s,t)$ for a particular value of $\s$ and for any $t \geq 0$,
the desired gradient $\frac{\partial J}{\partial \theta}(\x,\y,\theta)$ can be obtained by letting $t \to \infty$.
We show next that this is what the Recurrent Backpropagation algorithm does in the case where the initial state $\s$ is the fixed point $s_\theta^\x$.


\subsection{Recurrent Back-Propagation}
\label{sec:rec-backprop}

In order to compute the gradient of $J$ (Eq.~\ref{eq:objective-function}), the approach of \textit{Recurrent Backpropagation} \citep{Almeida87,pineda1987generalization}
is to compute $\frac{\partial L_\theta}{\partial \theta} \left( \x,\y,s_\theta^\x,t \right)$ for $t \geq 0$ iteratively.
We get the gradient in the limit $t \to \infty$ as a consequence of Eq.~\ref{eq:dL-dtheta-dJ-dtheta}, when the initial state $\s$ is the fixed point $s_\theta^\x$.
\footnote{The gradient $\frac{\partial L_\theta}{\partial \theta}\left( \x,\y,s_\theta^\x,t \right)$ represents the partial derivative of the function $L$ with respect to its first argument, evaluated at the fixed point $s_\theta^\x$.}

\begin{thm}[Recurrent Backpropagation]
    \label{thm:rec-backprop}
    Consider the process
    \begin{align}
        \overline{S}_t & := \frac{\partial L_\theta}{\partial s} \left( \x,\y,s_\theta^\x,t \right), \qquad t \geq 0, \\
        \overline{\Theta}_t & := \frac{\partial L_\theta}{\partial \theta} \left( \x,\y,s_\theta^\x,t \right), \qquad t \geq 0.
    \end{align}
    We call $(\overline{S}_t,\overline{\Theta}_t)$ the \textit{process of error derivatives}.
    It is the solution of the linear differential equation
    \begin{align}
        \label{eq:Cauchy-1}
        \overline{S}_0 & = \frac{\partial C_\theta}{\partial s} \left( \y,s_\theta^\x \right),      \\
        \label{eq:Cauchy-2}
        \overline{\Theta}_0 & = \frac{\partial C_\theta}{\partial \theta} \left( \y,s_\theta^\x \right), \\
        \label{eq:Cauchy-3}
        \frac{d}{dt} \overline{S}_t & = \frac{\partial \mu_\theta}{\partial s} \left( \x,s_\theta^\x \right)^T \cdot \overline{S}_t, \\
        \label{eq:Cauchy-4}
        \frac{d}{dt} \overline{\Theta}_t & = \frac{\partial \mu_\theta}{\partial \theta} \left( \x,s_\theta^\x \right)^T \cdot \overline{S}_t.
    \end{align}
    Moreover, as $t \to \infty$
    \begin{equation}
    	\overline{\Theta}_t \to \frac{\partial J}{\partial \theta}(\x,\y,\theta).
    \end{equation}
\end{thm}

Note that $\overline{S}_t$ takes values in the state space (space of the state variable $s$)
and $\overline{\Theta}_t$ takes values in the parameter space (space of the parameter variable $\theta$).

Theorem \ref{thm:rec-backprop} offers a way to compute the gradient $\frac{\partial J}{\partial \theta} \left( \x,\y,\theta \right)$.
In the first phase, the state variable $s$ follows the dynamics of Eq.~\ref{eq:free-dynamics} and relaxes to the fixed point $s_\theta^\x$.
Reaching this fixed point is necessary for the computation of the transpose of the Jacobian $\frac{\partial \mu_\theta}{\partial s} \left( \x,s_\theta^\x \right)^T$
which is required in the second phase.
In the second phase, the processes $\overline{S}_t$ and $\overline{\Theta}_t$ follow the dynamics determined by
Eq.~\ref{eq:Cauchy-1}-\ref{eq:Cauchy-4}.
As $t \to \infty$, we have that $\overline{\Theta}_t$
converges to the desired gradient $\frac{\partial J}{\partial \theta} \left( \x,\y,\theta \right)$.


\subsection{Temporal Derivatives of Neural Activities in Equilibrium Propagation Approximate Error Derivatives}

Two major objections against the biological plausibility of the recurrent backpropagation algorithm are that:
\begin{enumerate}
	\item it is not clear what the quantities $\overline{S}_t$ and $\overline{\Theta}_t$ would represent in a biological network, and
	\item it is not clear how dynamics such as those of Eq.~\ref{eq:Cauchy-1}-\ref{eq:Cauchy-4} for $\overline{S}_t$ and $\overline{\Theta}_t$ could emerge.
\end{enumerate}
By contrast, Equilibrium Propagation (section \ref{sec:equilibrium-propagation}) does not require specialized dynamics in the second phase.
Theorem \ref{thm} shows that the gradient $\frac{\partial J}{\partial \theta}(\x,\y,\theta)$ can be approximated by $\nu(\x,\y,\theta)$,
which can be itself estimated based on the first and second fixed points.
In this section we study the dynamics of the network in the second phase, from the first fixed point to the second fixed point.
Although Equilibrium Propagation does not compute explicit error derivatives, we are going to define a process $\left( \widetilde{S}_t,\widetilde{\Theta}_t \right)_{t \geq 0}$ as a function of the dynamics, and show that this process approximates the error derivatives $\left( \overline{S}_t,\overline{\Theta}_t \right)_{t \geq 0}$.

Let us denote by $S_\theta^\beta(\s,t)$ the state of the network at time $t \geq 0$
when it starts from an initial state $\s$ at time $t=0$ and follows the dynamics of Eq.~\ref{eq:augmented-dynamics}.
In particular, for $\beta=0$ we have $S_\theta^0(\s,t) = S_\theta^\x(\s,t)$.

The state of the network at the beginning of the second phase is the first fixed point $s_\theta^0$.
We choose as origin of time $t=0$ the moment when the second phase starts:
the network is in the state $s_\theta^0$ and the influence parameter takes on a small positive value $\beta \gtrsim 0$.
With our notations, the state of the network at time $t \geq 0$ in the second phase is $S_\theta^\beta \left( s_\theta^0,t \right)$.
At time $t=0$ the network is at the first fixed point, i.e. $S_\theta^\beta \left( s_\theta^0,0 \right) = s_\theta^0$, and as $t \to \infty$ the network's state converges to the second fixed point, i.e. $S_\theta^\beta(\s,t) \to s_\theta^\beta$.

Now let us define
\begin{align}
    \widetilde{S}_t      := & - \lim_{\beta \to 0} \frac{1}{\beta} \frac{\partial S_\theta^\beta}{\partial t} \left( s_\theta^0,t \right), \\
    \widetilde{\Theta}_t := & \frac{\partial C_\theta}{\partial \theta} \left( \y,s_\theta^0 \right) \nonumber \\
                            & - \lim_{\beta \to 0} \frac{1}{\beta} \frac{\partial \mu_\theta}{\partial \theta} \left( \x,s_\theta^0 \right)^T
    \cdot \left( S_\theta^\beta \left( s_\theta^0,t \right) - s_\theta^0 \right).
\end{align}
First of all note that $\widetilde{S}_t$ takes values in the state space and $\widetilde{\Theta}_t$ takes values in the parameter space.
From the point of view of biological plausibility,
unlike $\left( \overline{S}_t, \overline{\Theta}_t \right)$ in Recurrent Backpropagation, the process $\left( \widetilde{S}_t, \widetilde{\Theta}_t \right)$ in Equilibrium Propagation has a physiological interpretation.
Indeed $\widetilde{S}_t$ is simply the temporal derivative of the neural activity at time $t$ in the second phase (rescaled by a factor $\frac{1}{\beta}$).
As for $\widetilde{\Theta}_t$, the first term is zero in the case of the quadratic cost (Eq.~\ref{eq:quadratic-cost})
\footnote{If the cost function includes a regularization term of the form $\frac{1}{2} \norm{\theta}^2$, the first term is a backmoving force $-\lambda \theta$ modeling a form of synaptic depression.}
and the second term corresponds to the STDP-compatible weights change (Eq.~\ref{eq:parameter-change}) integrated between the initial state (the first fixed point) and the state at time $t$ in the second phase (and rescaled by a factor $\frac{1}{\beta}$).
For short, we call $(\widetilde{S}_t,\widetilde{\Theta}_t)$ the \textit{process of temporal derivatives}.

\begin{thm}
    \label{thm:temporal-derivatives}
    The process of temporal derivatives $(\widetilde{S}_t,\widetilde{\Theta}_t)$ satisfies
	\begin{align}
	    \label{eq:Cauchy-bis-1}
	    \widetilde{S}_0                   & = \frac{\partial C_\theta}{\partial s}        \left( \y,s_\theta^\x \right), \\
	    \label{eq:Cauchy-bis-2}
	    \widetilde{\Theta}_0              & = \frac{\partial C_\theta}{\partial \theta}   \left( \y,s_\theta^\x \right), \\
	    \label{eq:Cauchy-bis-3}
	    \frac{d}{dt} \widetilde{S}_t      & = \frac{\partial \mu_\theta}{\partial s}      \left( \x,s_\theta^\x \right)   \cdot \widetilde{S}_t, \\
	    \label{eq:Cauchy-bis-4}
	    \frac{d}{dt} \widetilde{\Theta}_t & = \frac{\partial \mu_\theta}{\partial \theta} \left( \x,s_\theta^\x \right)^T \cdot \widetilde{S}_t.
	\end{align}
	Furthermore, as $t \to \infty$
	\begin{equation}
		\widetilde{\Theta}_t \to \nu(\x,\y,\theta).
	\end{equation}
\end{thm}
Theorem \ref{thm:rec-backprop} and Theorem \ref{thm:temporal-derivatives} show that the processes $(\overline{S}_t,\overline{\Theta}_t)$ and $(\widetilde{S}_t,\widetilde{\Theta}_t)$ satisfy related differential equations.
The difference between the dynamics of Theorem \ref{thm:rec-backprop} and Theorem \ref{thm:temporal-derivatives} lies in Eq.~\ref{eq:Cauchy-3} and Eq.~\ref{eq:Cauchy-bis-3}.

As in Theorem \ref{thm}, the discrepancy between these processes is directly linked to the `degree of symmetry' of the Jacobian of $\mu_\theta$.
Again, an important particular case is the energy-based setting in which $\mu_\theta = - \frac{\partial E_\theta}{\partial s}$ for some scalar function $E_\theta(\x,s)$.
In this case $\frac{\partial \mu_\theta}{\partial s} = - \frac{\partial^2 E_\theta}{\partial s^2} = \left( \frac{\partial \mu_\theta}{\partial s} \right)^T$, and by Theorems \ref{thm:rec-backprop} and \ref{thm:temporal-derivatives} we get $\widetilde{S}_t = \overline{S}_t$ and $\widetilde{\Theta}_t = \overline{\Theta}_t$.
This result was stated and proved in \citet{scellier2017equivalence}.



\subsection{Proofs of Theorems \ref{thm:rec-backprop} and \ref{thm:temporal-derivatives}}

\begin{proof}[Proof of Theorem \ref{thm:rec-backprop}]
    First of all, by definition of $L$ (Eq.~\ref{eq:projected-cost}) we have $L_\theta(\s,0) = C_\theta(\s)$.
    Therefore the initial conditions (Eq.~\ref{eq:Cauchy-1} and Eq.~\ref{eq:Cauchy-2}) are satisfied:
    \begin{equation}
        \frac{\partial L_\theta}{\partial s} \left( s_\theta^0,0 \right)
        = \frac{\partial C_\theta}{\partial s} \left( s_\theta^0 \right)
    \end{equation}
    and
    \begin{equation}
        \frac{\partial L_\theta}{\partial \theta} \left( s_\theta^0,0 \right)
        = \frac{\partial C_\theta}{\partial \theta} \left( s_\theta^0 \right).
    \end{equation}

    Now we show that $\overline{S}_t = \frac{\partial L}{\partial s} \left( s^0,t \right)$ satisfies the differential equation (Eq.~\ref{eq:Cauchy-3}).
    We omit to write the dependence in $\theta$ to keep notations simple.
    As a preliminary result, we show that for all initial state $\s$ and time $t$ we have
    \footnote{Eq.~\ref{eq:proof1-1} is the Kolmogorov backward equation for a deterministic process.}
    \begin{equation}
        \label{eq:proof1-1}
        \frac{\partial L}{\partial t}(\s,t) = \frac{\partial L}{\partial s}(\s,t) \cdot \mu(\s).
    \end{equation}
    To this end note that (by definition of $L$ and $S^0$) we have for all $t$ and $u$
    \begin{equation}
        \label{eq:proof1-2}
        L \left( S^0(\s,u),t \right) = L(\s,t+u).
    \end{equation}
    The derivatives of the right-hand side of Eq.~\ref{eq:proof1-2} with respect to $t$ and $u$ are clearly equal:
    \begin{equation}
        \frac{d}{dt} L(\s,t+u) = \frac{d}{du} L(\s,t+u).
    \end{equation}
    Therefore the derivatives of the left-hand side of Eq.~\ref{eq:proof1-2} are equal too:
    \begin{align}
        \frac{\partial L}{\partial t}\left( S^0(\s,u),t \right) & = \frac{d}{du} L \left( S^0(\s,u),t \right) \\
        & = \frac{\partial L}{\partial s}\left( S^0(\s,u),t \right) \cdot \mu\left( S^0(\s,u),t \right).
    \end{align}
    Here we have used the differential equation of motion (Eq.~\ref{eq:free-dynamics}).
    Evaluating this expression for $u=0$ we get Eq.~\ref{eq:proof1-1}.
    Then, differentiating Eq.~\ref{eq:proof1-1} with respect to $s$, we get
    \begin{equation}
        \frac{\partial^2 L}{\partial t \partial s}(\s,t)
        = \frac{\partial^2 L}{\partial s^2}(\s,t) \cdot \mu(\s)
        + \left( \frac{\partial \mu}{\partial s}(\s) \right)^T \cdot \frac{\partial L}{\partial s}(\s,t).
    \end{equation}
    Evaluating this expression at the fixed point $\s = s^0$
    and using the fixed point condition $\mu \left( s^0 \right) = 0$ we get
    \begin{equation}
        \frac{d}{dt} \frac{\partial L}{\partial s} \left( s^0,t \right)
        = \left( \frac{\partial \mu}{\partial s} \left( s^0 \right) \right)^T \cdot \frac{\partial L}{\partial s} \left( s^0,t \right).
    \end{equation}
    Therefore $\frac{\partial L}{\partial s} \left( s^0,t \right)$ satisfies Eq.~\ref{eq:Cauchy-3}.

    Finally we prove Eq.~\ref{eq:Cauchy-4}.
    Differentiating Eq.~\ref{eq:proof1-1} with respect to $\theta$, we get
    \begin{align}
        \frac{\partial^2 L_\theta}{\partial t \partial \theta} \left( \s,t \right)
        & = \frac{\partial^2 L_\theta}{\partial s \partial \theta} \left( \s,t \right) \cdot \mu_\theta(\s) \\
        & + \left( \frac{\partial \mu_\theta}{\partial \theta}(\s) \right)^T \cdot \frac{\partial L_\theta}{\partial s}(\s,t).
    \end{align}
    Evaluating this expression at the fixed point $\s = s_\theta^0$ we get
    \begin{equation}
        \frac{d}{dt} \frac{\partial L_\theta}{\partial \theta} \left( s_\theta^0,t \right)
        = \left( \frac{\partial \mu_\theta}{\partial \theta} \left( s_\theta^0 \right) \right)^T \cdot \frac{\partial L_\theta}{\partial s} \left( s_\theta^0,t \right).
    \end{equation}
    Hence the result.
\end{proof}

\begin{proof}[Proof of Theorem \ref{thm:temporal-derivatives}]
	First of all, note that
	\begin{align}
		\left. \frac{\partial^2 S_\theta^\beta}{\partial \beta \partial t} \right|_{\beta=0} \left( s_\theta^0,t \right)
		= & \lim_{\beta \to 0} \frac{1}{\beta} \left( \frac{\partial S_\theta^\beta}{\partial t} \left( s_\theta^0,t \right) - \frac{\partial S_\theta^0}{\partial t} \left( s_\theta^0,t \right) \right) \\
		= & \lim_{\beta \to 0} \frac{1}{\beta} \frac{\partial S_\theta^\beta}{\partial t} \left( s_\theta^0,t \right).
	\end{align}
	This is because $S_\theta^0 \left( s_\theta^0,t \right) = s_\theta^0$ for every $t \geq 0$,
	so that $\frac{\partial S_\theta^0}{\partial t} \left( s_\theta^0,t \right) = 0$ at every moment $t \geq 0$.
	Furthermore
	\begin{align}
        & \frac{\partial C_\theta}{\partial \theta} \left( s_\theta^0 \right) -
		\frac{\partial \mu_\theta}{\partial \theta} \left( s_\theta^0 \right)^T
        \cdot \left. \frac{\partial S_\theta^\beta}{\partial \beta} \right|_{\beta=0} \left( s_\theta^0,t \right) \\
		= & \frac{\partial C_\theta}{\partial \theta} \left( s_\theta^0 \right) -
        \lim_{\beta \to 0} \frac{1}{\beta} \left(
        \frac{\partial \mu_\theta}{\partial \theta} \left( s_\theta^0 \right)^T
        \cdot \left( S_\theta^\beta \left( s_\theta^0,t \right) - s_\theta^0 \right)
        \right)
	\end{align}
	Thus, we have to show that the process $(\widetilde{S}_t,\widetilde{\Theta}_t)$ defined as
	\begin{align}
		\widetilde{S}_t      & := - \left. \frac{\partial^2 S_\theta^\beta}{\partial \beta \partial t} \right|_{\beta=0} \left( s_\theta^0,t \right), \\
        \label{eq:proof3-1}
		\widetilde{\Theta}_t & := \frac{\partial C_\theta}{\partial \theta} \left( s_\theta^0 \right) +
        \frac{\partial \mu_\theta}{\partial \theta} \left( s_\theta^0 \right)^T
        \cdot \left. \frac{\partial S_\theta^\beta}{\partial \beta} \right|_{\beta=0} \left( s_\theta^0,t \right)
	\end{align}
	satisfies Eq.~\ref{eq:Cauchy-bis-1}, Eq.~\ref{eq:Cauchy-bis-2}, Eq.~\ref{eq:Cauchy-bis-3} and Eq.~\ref{eq:Cauchy-bis-4}.

	First we prove the result for $\widetilde{S}_t$.
	We omit to write the dependence in $\theta$ to keep notations simple.
	The process $\left( S^\beta(\s,t) \right)_{t \geq 0}$ is solution of the differential equation
	\begin{equation}
		\label{eq:proof3-2}
		\frac{\partial S^\beta}{\partial t}(\s,t) = \mu^\beta \left( S^\beta(\s,t) \right).
	\end{equation}
	with initial condition $S^\beta(\s,0) = \s$.
	Differentiating Eq.~\ref{eq:proof3-2} with respect to $\beta$,
	we get the following equation for the process $\frac{\partial S^\beta}{\partial \beta}(\s,t)$:
	\begin{equation}
		\frac{d}{dt} \frac{\partial S^\beta}{\partial \beta}(\s,t) =
		\frac{\partial \mu^\beta}{\partial \beta} \left( S^\beta(\s,t) \right)
		+ \frac{\partial \mu^\beta}{\partial s} \left( S^\beta(\s,t) \right) \cdot \frac{\partial S^\beta}{\partial \beta}(\s,t).
	\end{equation}
	Evaluating at $\beta=0$, taking $\s = s^0$ as an initial state and using the fact that $S^0 \left( s^0,t \right) = s^0$, we get
	\begin{equation}
		\label{eq:proof3-3}
		\frac{d}{dt} \left. \frac{\partial S^\beta}{\partial \beta} \right|_{\beta=0} \left( s^0,t \right) =
		- \frac{\partial C}{\partial s} \left( s^0 \right)
		+ \frac{\partial \mu}{\partial s} \left( s^0 \right) \cdot
		\left. \frac{\partial S^\beta}{\partial \beta} \right|_{\beta=0} \left( s^0,t \right).
	\end{equation}
	Since $S^\beta \left( \s,0 \right) = \s$ is independent of $\beta$, we have $\frac{\partial S^\beta}{\partial \beta} \left( \s,0 \right) = 0$.
	Therefore, evaluating Eq.~\ref{eq:proof3-3} at $t=0$, we get the initial condition (Eq.~\ref{eq:Cauchy-bis-1}):
	\begin{equation}
		\left. \frac{\partial^2 S^\beta}{\partial t \partial \beta} \right|_{\beta=0} \left( s^0,0 \right) =
		- \frac{\partial C}{\partial s} \left( s^0 \right).
	\end{equation}
	Moreover, differentiating Eq.~\ref{eq:proof3-3} with respect to time we get Eq.~\ref{eq:Cauchy-bis-3}:
	\begin{equation}
		\frac{d}{dt} \left. \frac{\partial^2 S^\beta}{\partial t \partial \beta} \right|_{\beta=0} \left( s^0,t \right) =
		\frac{\partial \mu}{\partial s} \left( s^0 \right) \cdot
		\left. \frac{\partial^2 S^\beta}{\partial t \partial \beta} \right|_{\beta=0} \left( s^0,t \right).
	\end{equation}

	Now we prove the result for $\widetilde{\Theta}_t$.
	Evaluating Eq.~\ref{eq:proof3-1} at time $t=0$ we get the initial condition (Eq.~\ref{eq:Cauchy-bis-2})
	\begin{equation}
		\widetilde{\Theta}_0 = \frac{\partial C_\theta}{\partial \theta} \left( s_\theta^0 \right).
	\end{equation}
	Moreover, differentiating Eq.~\ref{eq:proof3-1} with respect to time we get Eq.~\ref{eq:Cauchy-bis-4}:
	\begin{equation}
		\frac{d}{dt} \widetilde{\Theta}_t
		= - \frac{\partial \mu_\theta}{\partial \theta} \left( s_\theta^0 \right)^T
		\cdot \left. \frac{\partial^2 S_\theta^\beta}{\partial t \partial \beta} \right|_{\beta=0} \left( s_\theta^0,t \right).
	\end{equation}
	Hence the result.
\end{proof}

\section{Experiments and Implementation of the Model}
\label{sec:implementation}

\begin{figure*}[h!]
\begin{center}
$\begin{array}{|c|cc|c|c|cccc|}
\hline
	\hbox{Architecture} & \hbox{Iterations}      & \hbox{Iterations}     & \epsilon & \beta & \alpha_1 & \alpha_2 & \alpha_3 & \alpha_4 \\
	                    & (\hbox{first phase})   & (\hbox{second phase}) &          &       &          &          &          &         \\
\hline
    784-512-512-10      & 200 & 100 & 0.001 & 1.0 & 0.4 & 0.1 & 0.01 & --\\
	784-512-512-512-10  & 200 & 100 & 0.001 & 1.0 & 1.0 & 0.1 & 0.04 & 0.002 \\
\hline
\end{array}$
  \captionof{table}{Hyperparameters for both the 2- and 3-layer networks trained on the MNIST dataset, as described in Appendix \ref{sec:implementation}. The objective function is optimized: the training error decreases to 0.00\%. The generalization error lies between 2\% and 3\% depending on the architecture. The learning rate $\epsilon$ is used
    for iterative inference (Eq. \ref{eq:udpate_dyn}).
    $\beta$ is the value of the influence parameter in the second phase.
    $\alpha_k$ is the learning rate for updating the parameters in layer $k$.}
\label{table:hyperparameters}
\end{center}
\end{figure*}

\begin{table*}[ht]
  \centering
  \begin{tabular}{|c|c|c|c|c|c|}
	\hline
    \textbf{Operation} & \textbf{Kernel} & \textbf{Strides} & \textbf{Feature Maps}& \textbf{Non Linearity} \\
    Convolution & 5 x 5 & 1 & 32 &   Relu  \\
    Convolution & 5 x 5 & 1 & 64 &  Relu \\
	\hline
  \end{tabular}
  \caption{Hyperparameters for MNIST CNN experiments.}
  \label{tab:Celeba_param}
\end{table*}

Our model is a recurrently connected neural network without any constraint on the feedback weight values (unlike models such as the Hopfield network).
We train multi-layered networks with 2 or 3 hidden layers on the MNIST task, with no skip-layer connections and no lateral connections within layers, as in Figure \ref{fig:network} (though the theory presented in this paper applies to any network architecture).

Rather than performing weight updates at all time steps in the second phase, we perform a single update at the end of the second phase:
\begin{equation}
	\label{eq-update-experiments}
	\Delta W_{ij} \propto \frac{\partial \mu_\theta}{\partial W_{ij}} \left( \x,s_\theta^0 \right)^T \cdot \frac{s_\theta^\beta - s_\theta^0}{\beta}.
\end{equation}

The predicted value (given the input $\x$) is read on the last layer at the first fixed point $s_0^0$ at the end of the first phase.
The predicted value $\widehat{y}$ is the index of the output unit whose activation is maximal among the $10$ output units:
\begin{equation}
	\widehat{y} := \underset{i}{\arg \max} \; s_{0,i}^0.
\end{equation}

{\bf Implementation of the neuronal dynamics.}
We start by clamping $\x$ to the data values.
Then we use the Euler method to implement Eq.~\ref{eq:augmented-dynamics}.
The naive method is to discretize time into short time lapses of duration $\epsilon$ and to update the state variable $s$ iteratively thanks to
\begin{equation}
	\label{eq:grad-descent}
	s \leftarrow s - \epsilon \; \mu_\theta^\beta(\x,\y,s).
\end{equation}
For our experiments, we choose the hard sigmoid activation function $\rho(s_i) = 0 \vee s_i \wedge 1$, where $\vee$ denotes the max and $\wedge$ the min.
To address stability issues, we restrict the range of values for the neurons' states and clip them between $0$ and $1$.
Thus, rather than the standard Euler method (Eq. \ref{eq:grad-descent}), we use a slightly different update rule for the state variable $s$:
\begin{equation}
	\label{eq:udpate_dyn}
	s \leftarrow 0 \vee \left( s - \epsilon \; \mu_\theta^\beta(\x,\y,s) \right) \wedge 1.
\end{equation}

We use different learning rates for different layers in our experiments.
The hyperparameters chosen for each model are shown in Table \ref{table:hyperparameters}.
We initialize the weights according to the Glorot-Bengio initialization~\citep{GlorotAISTATS2010-small}.
For efficiency of the experiments, we use minibatches of 20 training examples.

We were also able to train on MNIST using a Convolutional Neural Network (CNN).
We got around 2\% generalization error.
The hyperparameters chosen to train this Convolutional Neural Network are shown in Table \ref{tab:Celeba_param}.

\begin{figure}[h!]
	\centering
	\captionsetup{width=.8\linewidth}
	\includegraphics[width=.6\linewidth]{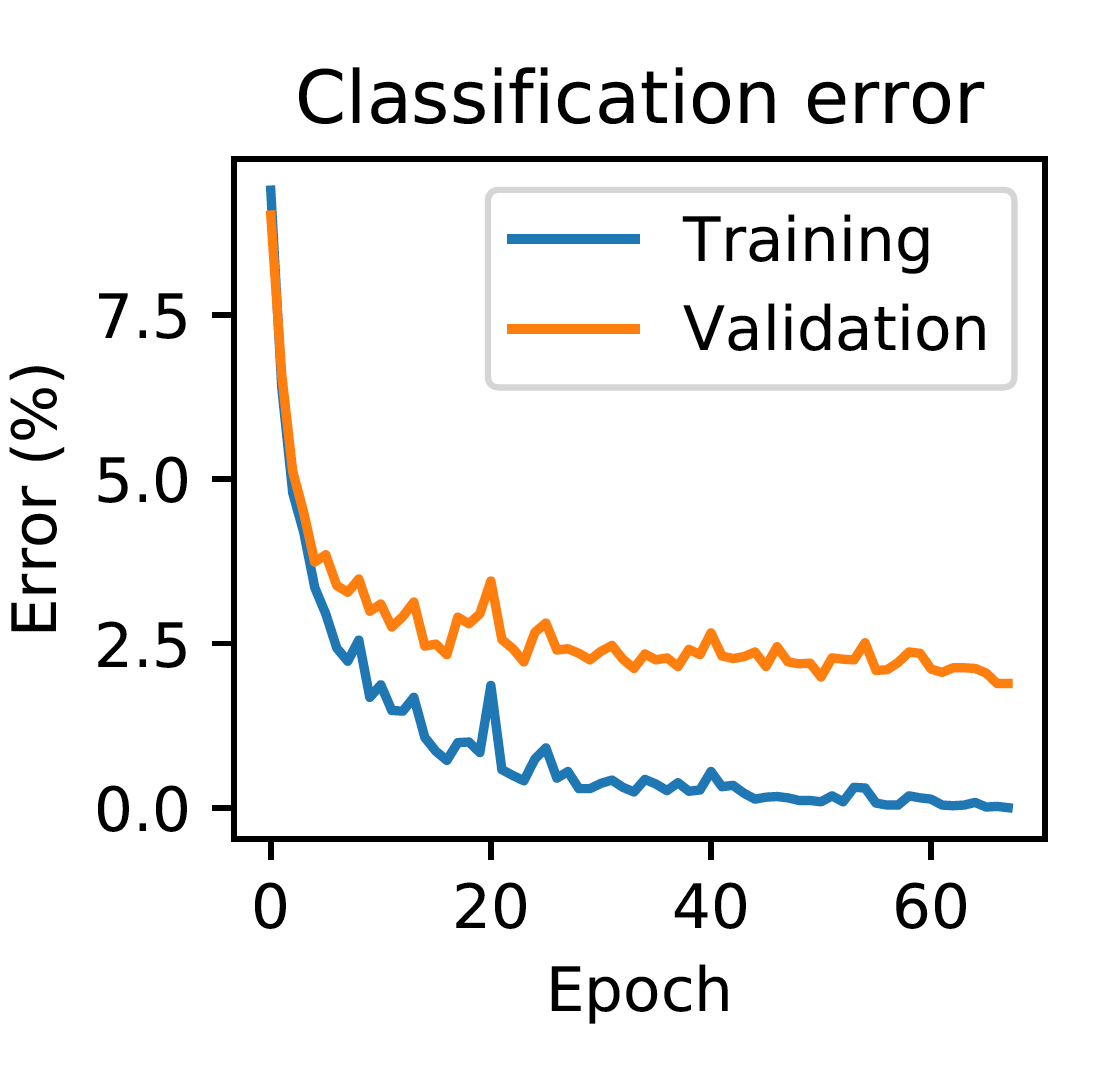}
	\caption{Experiments on the MNIST dataset.
	}
	\label{fig:network}
\end{figure}

\end{document}